\newcommand{\papertype}{neurodata}
\ifnum\pdfstrcmp{\papertype}{neurodata}=0
\documentclass{article}
\usepackage{neurodata/neurodata}
\usepackage{lipsum}
\else
\documentclass[10pt,twocolumn,letterpaper]{article}
\usepackage[accsupp]{axessibility}  
\usepackage{iccv}
\usepackage{times}
\usepackage{epsfig}
\usepackage{graphicx}
\usepackage{amsmath}
\usepackage{amssymb}
\fi

\usepackage[utf8]{inputenc} 
\usepackage[T1]{fontenc}    
\usepackage{url}            
\usepackage{booktabs}       
\usepackage{amsfonts}       
\usepackage{nicefrac}       
\usepackage{microtype}      
\usepackage{xcolor}         

\usepackage{multirow}
\usepackage{array}
\usepackage{graphicx}
\usepackage{amsthm}
\usepackage{mathtools}
\usepackage{algorithm} 
\usepackage{algpseudocode} 
\usepackage[createShortEnv]{proof-at-the-end}
\newtheorem{theorem}{Theorem}[section]
\newtheorem{lemma}{Lemma}[section]
\newtheorem{corollaryinner}{Corollary}[theorem] 
\newenvironment{corollary}[1][]
{
\if\relax\detokenize{#1}\relax
\else
\ifcsname #1-used\endcsname
  \expandafter\xdef\csname #1-used\endcsname{\the\numexpr\csname #1-used\endcsname+1}%
\else
  \expandafter\gdef\csname #1-used\endcsname{1}%
\fi
\renewcommand{\thecorollaryinner}{\ref{#1}.\csname #1-used\endcsname}%
\fi
\corollaryinner
}
{\endcorollaryinner}
\newtheorem{definition}{Definition}[section]
\graphicspath{ {images/} }
\usepackage{ifthen} 

\ifthenelse{\equal{\papertype}{neurodata}}{
}{
\newtheorem{remark}{Remark}
\usepackage[pagebackref=true,breaklinks=true,letterpaper=true,colorlinks,bookmarks=false]{hyperref}

\iccvfinalcopy 

\def\iccvPaperID{10905} 

\ificcvfinal\pagestyle{empty}\fi
}

\begin{document}

\title{Why do networks have inhibitory/negative connections?}

\ifthenelse{\equal{\papertype}{neurodata}}{
\author[1,5]{Qingyang ~Wang\thanks{qwang88@jhu.edu}}
\author[2,5,7]{Michael A.~Powell}
\author[2,5]{Ali ~Geisa}
\author[4,5]{Eric ~Bridgeford}
\author[6]{Carey E. Priebe}
\author[1,2,3,4,5]{Joshua T.~Vogelstein}
\affil[1]{Department of Neuroscience, Johns Hopkins University}
\affil[2]{Department of Biomedical Engineering, Johns Hopkins University}
\affil[3]{Institute for Computational Medicine, Kavli~Neuroscience~Discovery Institute, Johns Hopkins University}
\affil[4]{Department of Biostatistics, Johns Hopkins University}
\affil[5]{Center for Imaging Science, Johns Hopkins University}
\affil[6]{Department of Applied Mathematics and Statistics, Johns Hopkins University, United States}
\affil[7]{Current location: United States Military Academy, Department of Mathematical Sciences, West Point NY US}
}{
\renewcommand*{\thefootnote}{\fnsymbol{footnote}}
\author{Qingyang ~Wang\footnotemark[1],\;Michael A.~Powell\footnotemark[4],\;Ali ~Geisa,\;Eric ~Bridgeford,\;Carey E. Priebe,\;Joshua T.~Vogelstein\\
Johns Hopkins University\\
}
}

\maketitle
\ifthenelse{\equal{\papertype}{neurodata}}{}{
\ificcvfinal\thispagestyle{empty}\fi
\footnotetext[1]{Correspondence: qwang88@jhu.edu}
\footnotetext[4]{Current address: United States Military Academy, Department of Mathematical Sciences, West Point NY US}
}
\renewcommand*{\thefootnote}{\arabic{footnote}}

\begin{abstract}
    Why do brains have inhibitory connections? 
    Why do deep networks have negative weights?
    We propose an answer from the perspective of representation capacity. 
    We believe representing functions is the primary role of both (i) the brain in natural intelligence, and (ii) deep networks in artificial intelligence. 
    Our answer to why there are inhibitory/negative weights is: to learn more functions. 
    We prove that, in the absence of negative weights, neural networks with non-decreasing activation functions are \textit{not} universal approximators. 
    While this may be an intuitive result to some, to the best of our knowledge, there is no formal theory, in either machine learning or neuroscience, that demonstrates \textit{why} negative weights are crucial in the context of representation capacity. 
    Further, we provide insights on the geometric properties of the representation space that non-negative deep networks cannot represent.
    We expect these insights will yield a deeper understanding of more sophisticated inductive priors imposed on the distribution of weights that lead to more efficient biological and machine learning.  
\end{abstract}


\section{Introduction}\label{sec:intro}
Are inhibitory connections necessary for a brain to function? Many studies on mammals answer yes: a balanced excitatory/inhibitory (E/I) ratio is essential for memory \cite{Lim2013}, unbalanced E/I lead to either impulsive or indecisive behaviors \cite{Lam2022}, such balance is closely tracked throughout learning \cite{Vogels2011, Sukenik2021}, and imbalance is hypothesized to be the driving force behind epilepsy \cite{Cohen2002, Huberfeld2011, Truccolo2011}. These simulation results and disease studies provide compelling evidence that E/I balance is necessary for brains to function stably. 

Intriguingly, when neurons first came into existence in the long history of evolution, they were exclusively excitatory \cite{Kristan2016}. The Cnidarian jellyfish has a well-defined nervous system that is made of sensory neurons, motor neurons, and interneurons. Different from mammals, their synaptic connections are exclusively excitatory. Even though these jellyfish are not equipped with inhibitory neurons, they are perfectly capable of performing context-dependent behaviors: when they feed, they swim slowly through a weak, rhythmic contraction of the whole bell; when they sense a strong mechanical stimulus, they escape through a rapid, much stronger contraction \cite{Kristan2016}. Cnidarian jellyfish behave adaptively, not through sophisticated excitatory-inhibitory circuits, but instead by modifying voltage-gated channel properties (i.e., conductance). Cnidarian jellyfish prove to us that without inhibitory connections the brain can still function, albeit likely through alternative mechanisms. This re-raises the fundamental question: are inhibitory connections necessary for brains to function?

In this paper, we explore the necessity of inhibitory connections from the perspective of representation capacity. Instead of viewing the brain as a dynamical system to discuss its functional stability, we think about the brain as a feedforward network capable of representing functions. For a certain network structure, we are interested in characterizing the repertoire of functions such networks are capable of representing. Specifically, to understand the importance of inhibitory connections in networks' representation capacity, we ask what functions can non-negative networks represent? We do so with the help of Deep Neural Networks (DNNs). DNNs allow us to work at a level of abstraction where we can stay focused on the connectivity between computation units; they also allow us to completely take out inhibitory connections by setting all weights to be non-negative. They further allow us to build upon the well-celebrated DNN theoretical result that DNNs are universal approximators \cite{Hornik1989, Cybenko1989, Hornik1991} (Theorem~\ref{lemma: hornik}). We prove in this paper that DNNs with all non-negative weights are \emph{not} universal approximators (sec~\ref{sec:app}). We further prove three geometric properties of the representation space for non-negative DNNs (sec~\ref{sec:geometry}). These results show that networks without negative connections not only lose universality; in fact, they have extremely limited representation space. We further extend our results to convolutional neural networks (CNNs) and other structural variants (sec~\ref{sec:extension}). Such theoretical results serve as a plausible explanation for why purely excitatory nervous systems, along the history of evolution, were largely overtaken by brains containing both inhibitory and excitatory connections; the simpler systems may be able to perform interesting functions, but they also have limited representation capacity. Our work thus concludes that inhibitory connections are necessary for a brain to represent more functions. 

\section{Related works}
\subsection{Optical neural network}
The optical neural network literature \cite{Dickey1993, DeLaurentis1994} has some discussion on the representation capacity of non-negative DNNs; however, the existing discussion is restricted to only non-negative functions, i.e. $f(x)\geqslant 0,\; \forall x\in \mathcal X$. Our work does not impose any requirement on the network output since we allow the bias terms to take any value and the output of the network may be negative (dependent on the activation function choice). 
\subsection{Partially monotone networks}
Partially monotone networks exploited the monotone nature of non-negative sub-networks to solve problems where certain input features are known to stay monotone. One notable application is in stock prediction~\cite{Hennie2010}. An independent stream of research that falls into the category of partially monotone networks is input-convex networks, which are built for fast inference~\cite{Amos2017}. Our work is the first one using the fact that non-negative DNNs are order-preserving monotone functions to explicitly prove that they are not universal approximators; we are also the first pointing out its geometric consequences and their implications in neuroscience, as well as the first to extend these theoretical results to CNNs. 

\section{$\text{DNN}^{+}$s are not universal approximators}\label{sec:app}
DNNs are universal approximators \cite{Hornik1989, Cybenko1989, Hornik1991} (paraphrased in Theorem ~\ref{lemma: hornik}): given a large enough network, it can learn (represent) a continuous function arbitrarily well. Throughout this paper, we are not concerned about the learning process \textit{per se}, but more about its performance upper bound. Given unlimited resources (e.g., time, energy, compute units), can networks of a certain structure learn a function at all? If it can, we say such networks can \textbf{represent} such a class of functions. We show below that non-negative DNNs (all weights non-negative, abbreviated as $\text{DNN}^{+}$) are \emph{not} universal approximators (Theorem~\ref{lma3}); thus, having both positive and negative weights is necessary for universal approximation. We start by defining the problem and then follow with our key observation: $\text{DNN}^{+}$s are composed of order-preserving monotone functions (Definition~\ref{def:e-w}, Lemma~\ref{lma1}); therefore, non-negative DNNs are all order-preserving monotone functions (Theorem~\ref{lma2}). With this key property of $\text{DNN}^{+}$ in mind, we can then show that $\text{DNN}^{+}$s cannot solve XOR and are therefore not universal approximators. 

\subsection{$\textbf{DNN}^\mathbf{{+}}$}
\begin{definition}[DNN]\label{def:DNN}
A feedforward fully-connected (FC) neural network architecture is given by the tuple $(\mathbf{\Phi}, n_0, n_1, n_2, \dots, n_L)$, where $L \in \mathbb{N}^{+}$ is the number of layers, $n_l \in\mathbb{N}$ is the number of units of layer $l$, $l \in [L]$; $\mathbf{\Phi}: \mathbb{R}^{n_{l}} \to \mathbb{R}^{n_{l}}$ is a point-wise nonlinear non-decreasing function, and it defines the non-linear component of layer $l$. We call $\mathbf{\Phi}$ the activation function. The linear component of the transformation from layer $l-1$ to $l$ is given by the weight matrix $\mathbf{W}^{(l)} \in \mathbb{R}^{n_{l} \times n_{(l-1)}}$ and bias vector $\mathbf{b}^{(l)} \in \mathbb{R}^{n_l}$. Collectively, the function given by a DNN with the above architecture tuple $(\mathbf{\Phi}, n_0, n_1, n_2, \dots, n_L)$ takes input $\mathbf{x}\in\mathbb{R}^{n_0}$ and is defined as following:
\ifthenelse{\equal{\papertype}{neurodata}}{
\begin{multline}
F:\mathbb{R}^{n_0}\to\mathbb{R}^{n_{L}}, \quad F(\mathbf{x}) = \mathbf{\Phi}(\mathbf{W}^{(L)}\mathbf{\Phi}(\mathbf{W}^{(L-1)} \dots \mathbf{\Phi}(\mathbf{W}^{(1)}\mathbf{x} + \mathbf{b}^{(1)}) \dots + \mathbf{b}^{(L-1)}) + \mathbf{b}^{(L)}), \quad \mathbf{x}\in\mathbb{R}^{n_0}.
\end{multline} 
}{
\begin{multline}
F:\mathbb{R}^{n_0}\to\mathbb{R}^{n_{L}}, \quad F(\mathbf{x}) = \mathbf{\Phi}(\mathbf{W}^{(L)}\mathbf{\Phi}(\mathbf{W}^{(L-1)} \\
\dots \mathbf{\Phi}(\mathbf{W}^{(1)}\mathbf{x} + \mathbf{b}^{(1)}) \dots + \mathbf{b}^{(L-1)}) + \mathbf{b}^{(L)}), \quad \mathbf{x}\in\mathbb{R}^{n_0}.
\end{multline}  
}
\end{definition}

Non-negative DNNs ($\text{DNN}^{+}$s) are DNNs with all weights non-negative (without any constraint on the bias terms).

\begin{definition}[$\text{DNN}^{+}$]\label{def:dnnplus}
$\text{DNN}^{+}$ is a DNN with all non-negative weights. $F^{+}:\mathbb{R}^{n_0}\to\mathbb{R}^{n_{L}}$, 
\begin{align*}
    F^{+}(\mathbf{x}) = & \mathbf{\Phi}(\mathbf{W}^{(L)} \dots \mathbf{\Phi}(\mathbf{W}^{(l)} \dots \mathbf{x} \ldots + \mathbf{b}^{(l)}) \dots + \mathbf{b}^{(L)})\\
    & \forall l\in[L], \quad \mathbf{W}^{(l)}\geq \mathbf{0}\footnotemark, \quad \mathbf{b}^{(l)}\in\mathbb{R}^{n_l}
\end{align*}
\end{definition}
\footnotetext{$\forall i\in[n_{l-1}], j\in[n_l], \;w_{ij}\geq 0$}

\begin{remark}[Activation function $\mathbf{\Phi}$]
We only impose one constraint to the activation function $\mathbf{\Phi}$ in addition to it being non-linear: it must be non-decreasing. Popular choices of activation function fall into this category (e.g., ReLU, leaky ReLU, sigmoid, tanh, etc.). 

Throughout the paper, we boldface $\mathbf{\Phi}$ to emphasize it is a point-wise function on the vector space and to distinguish it from its base form $\phi:\mathbb{R}\to\mathbb{R}$. By point-wise, we mean $\mathbf{\Phi}(\mathbf{x}) = (\phi(x_1), \phi(x_2), ..., \phi(x_{n_l}))$. 

Two examples of $\phi$ are given below:
\begin{align*}
& ReLU: \phi(x) = 
    \begin{cases}
    0 \quad x<0\\
    x \quad x\geqslant0
    \end{cases} \qquad \mathrm{sigmoid}(x) = \frac{1}{1+e^{-x}}
\end{align*}
\end{remark}



\begin{remark}[Extensions]
Definition~\ref{def:DNN} presents the most classical form of DNN as defined in the original universal approximator papers \cite{Hornik1989, Cybenko1989, Hornik1991}. Such DNNs are also called multi-layer perceptrons (MLPs). Since the 1990s, variants of the basic operations have evolved. We will discuss some of these extensions in section~\ref{sec:extension} and then prove that all of our results in sections~\ref{sec:app} and~\ref{sec:geometry} generalize to these extended DNN architectures. 
\end{remark}

\subsection{$\textbf{DNN}^\mathbf{+}$s are order-preserving monotone functions} 
We are particularly interested in a concise description of the class of functions representable by $\text{DNN}^{+}$. Intuitively, by constraining all weights to be non-negative, the possible transformation operations of the linear components of the functions are limited, which cannot be overcome by translations given by the bias terms, even though the bias terms can change within the full range of $\mathbb{R}$. This means that rotations or reflections that result in a change of orthant are impossible. One important consequence of these limitations is that the composite function becomes a non-decreasing monotone function in high-dimensional space. Below, we formally prove these ideas by first defining a partial ordering in the high-dimensional vector space (Definition~\ref{def:order}), and then by proving $\text{DNN}^{+}$s are order-preserving functions (Theorem~\ref{lma2}). 

\begin{definition}[$\preceq$, partial ordering on $\mathbb{R}^n$]\label{def:order}
For any pair of $\mathbf{x,y}\in\mathbb{R}^n$, we define a partial ordering\footnote{This is a standard order induced by the positive cone ${\mathbb{R}_+}^n$.}
\begin{equation*}
   \mathbf{x} \preceq \mathbf{y} \quad \equiv \quad  x_k \leqslant y_k,\; \forall k\in[n]  . 
\end{equation*}
\end{definition}

\begin{definition}[order-preserving monotone function, $\mathbb{R}^m\to\mathbb{R}^n$]\label{def:e-w}
For a function $T: \mathbb{R}^m\to\mathbb{R}^n$, we say $T$ is a \textbf{order-preserving monotone function} when for any $ \mathbf{x, y} \in \mathbb{R}^m$
\begin{equation*}
    \mathbf{x} \preceq \mathbf{y} \implies T(\mathbf{x}) \preceq T(\mathbf{y}). 
\end{equation*}
\end{definition}

\begin{lemmaE}[closure of order-preserving monotone functions under sum, positive scaling, and translation][end]\label{lma1}
Given an affine transformation $T: \mathbb{R}^{m} \to \mathbb{R}^n$ with a weight matrix $\mathbf{W}$ of all non-negative entries:
\begin{equation*}
    T(\mathbf{x}) = \mathbf{W}\mathbf{x} + \mathbf{b}, \; \mathbf{x} \in \mathbb{R}^m, W \in \mathbb{R}_{\geq 0}^{n \times m}, \mathbf{b} \in \mathbb{R}^n.
\end{equation*}
We observe such $T$ is a monotone function (Def~\ref{def:e-w}). That is, $\forall \mathbf{x}, \mathbf{y} \in \mathbb{R}^m$, 
\begin{equation*}
    \mathbf{x}\preceq \mathbf{y} \implies T(\mathbf{x}) \preceq T(\mathbf{y}). 
\end{equation*}
\end{lemmaE}

\begin{proofE}
    By definition~\ref{def:order}, it suffices to prove $\forall i\in[n],\; T_i(\mathbf{x}) \leqslant T_i(\mathbf{y})$. Let ${w_{ij}}$ be the $i$th row and $j$th column component of ${\mathbf{W}}$:
\begin{align*}
    T_i(\mathbf{x})-T_i(\mathbf{y}) & = ({\sum_{j\in[m]}{w_{ij}x_{j}}}+ b_i) - ({\sum_{j\in[m]}{w_{ij}y_{j}}}+ b_i)\\
    & = {\sum_{j\in[m]}{w_{ij}(x_{j}-y_{j})}}\\
    & \leqslant 0
\end{align*}
\end{proofE}

Since the activation function of $\text{DNN}^{+}$ is also a non-decreasing monotone function, we immediately have $F^+$ is also an order-preserving function because it is the composite of order-preserving functions. 

\begin{theoremE}[$\text{DNN}^{+}$ is order-preserving monotone function][end]\label{lma2}
A $\text{DNN}^{+}$ $F^{+}:\mathbb{R}^{n_0}\to\mathbb{R}^{n_{L}}$ is an order-preserving monotone function (Defn~\ref{def:e-w}). That is, any $ \mathbf{x, y} \in \mathbb{R}^{n_0}$, 
\begin{equation*}
    \mathbf{x} \preceq \mathbf{y} \implies F^{+}(\mathbf{x}) \preceq F^{+}(\mathbf{y}). 
\end{equation*}
\end{theoremE}
\begin{proofE}
$\text{DNN}^{+}$ is a composition of monotone functions: layers of non-decreasing activation functions $\mathbf{\Phi}$ and affine transformation with all non-negative weights $T:\mathbb{R}^m\to\mathbb{R}^n$ (Lemma~\ref{lma1}). By closure of monotone function under compositionality, we have that $\text{DNN}^{+}$ is an order-preserving monotone function. 
\end{proofE}

\subsection{$\textbf{DNN}^\mathbf{+}$s are not universal approximators} 
Intuitively, $\text{DNN}^{+}$ being an order-preserving monotone function means it cannot solve any classification problems where order reversal is required. We will use this idea to prove $\text{DNN}^{+}$ cannot solve XOR, and thus they are not universal approximators (Theorem~\ref{lma3}). Below we paraphrase the universal approximator theorem from Hornik's 1989 paper.

\begin{theoremE}[Multilayer feedforward networks are universal approximators (paraphrased from Hornik 1989 \cite{Hornik1991})][]\label{lemma: hornik}
Let $C(K, \mathbb{R}^n)$ denotes the set of continuous functions from a compact domain $K \subseteq \mathbb{R}^m$ to $\mathbb{R}^n$. Then for every $m\in\mathbb{N}, n\in\mathbb{N}$, for all $K \subseteq \mathbb{R}^m$, $f\in C(K, \mathbb{R}^n)$, and for all $\epsilon>0$, there exists a DNN $F$ (Definition~\ref{def:DNN}) such that
\begin{equation*}
    \sup_{\mathbf{x}\in K}|f(\mathbf{x})-F(\mathbf{x})|<\epsilon
\end{equation*}
\end{theoremE}

One may notice, in the original universal approximator proofs \cite{Hornik1989, Cybenko1989, Hornik1991}, activation functions were required to be continuous, bounded, and non-constant. Thus for DNNs with non-decreasing monotone activation functions, they are all universal approximators. However, this is no longer true when all weights are constrained to be non-negative, as we will prove below by showing a counter example: $\text{DNN}^{+}$ cannot solve XOR. 

\begin{definition}[XOR-continuous]\label{def:XOR-con}
Let $x \in [0,1]^{2}$, $f: [0,1]^2 \to \mathbb{R}$
\begin{equation}
f(x_1,x_2) = x_1+x_2-2x_1x_2
\end{equation}
Note: f(0,0)=f(1,1)=0; f(1,0)=f(0,1)=1. 
\end{definition}

\begin{remark}[XOR-discontinuous]\label{def:XOR}
We also define a discontinuous version of XOR for the readers' convenience. Such a definition variation does not alter our result (Theorem~\ref{lma3}) in any way.   

Let $x \in [0,1]^{2} \: \mathrm{and} \: t_1, t_2 \in (0,1)$, $f: [0,1]^2 \to \{0,1\}$
\begin{equation}
 f(x) =
    \begin{cases}
      0, [x_1 \geq t_1 \: \mathrm{and} \: x_2 \geq t_2] \: \mathrm{or} \: [x_1 < t_1 \: \mathrm{and} \: x_2 < t_2]\\
      1, [x_1 \geq t_1 \: \mathrm{and} \: x_2 < t_2] \: \mathrm{or} \: [x_1 < t_1 \: \mathrm{and} \: x_2 \geq t_2].\\
    \end{cases}      
\end{equation}    
\end{remark}

\begin{theorem}[$\text{DNN}^{+}$ is not a universal approximator]
\label{lma3}
$\text{DNN}^{+}$s are not universal approximators. That is, there exists a combination of $(m,n,K,f)$ where $m\in\mathbb{N}, n\in\mathbb{N}$, $K \subseteq \mathbb{R}^m$, $f\in C(K, \mathbb{R}^n)$, and exists an $\epsilon>0$ such that for all $\text{DNN}^{+}$ $F^+$ (Definition~\ref{def:dnnplus}), 
\begin{equation*}
    \exists \;\mathbf{x}\in K, \text{such that}\; |f(\mathbf{x})-F^+(\mathbf{x})|\geqslant\epsilon
\end{equation*} 
\end{theorem}
\begin{proof}
It suffices to find one combination $(m,n,K,f)$ such that \emph{no} $F^+$ could approximate $f$ arbitrarily well. XOR, $(m=2,n=1,K=[0,1]^2,f=\text{Defn~\ref{def:XOR-con}})$, is such a combination, as we will prove below in Corollary~\ref{cor:XOR}.

\begin{corollary}[lma3][$\text{DNN}^{+}$ cannot approximate XOR]\label{cor:XOR} 
For $f$ as defined in Definition~\ref{def:XOR-con}, $\exists\;\epsilon>0$ such that $\forall \; F^+$ (Definition~\ref{def:dnnplus}), $\exists \;\mathbf{x}\in \mathbb{R}^2, \text{such that}\; |f(\mathbf{x})-F^+(\mathbf{x})|\geqslant\epsilon$
\end{corollary}
\begin{proof}
We will prove by contradiction. Assume $\text{DNN}^{+}$ $F^+$ can approximate XOR $f$ well, only with an error term of $\epsilon>0$ for all $\mathbf{x}\in\mathbb{R}^2$. 

Now lets consider three points $(0,0), (1,0), (1,1)$, we must have
\begin{align}
    |F^{+}(0, 0) - f(0,0)| & < \epsilon, \\
    |F^{+}(1, 0) - f(1,0)| & < \epsilon, \\
    |F^{+}(1, 1) - f(1,1)| & < \epsilon. 
\end{align}
From XOR definition~\ref{def:XOR-con}, we also have f(0,0)=f(1,1)=0, f(1,0)=f(0,1)=1, therefore, 
\begin{align*}
    & |F^{+}(0, 0)| < \epsilon, \\
    & |F^{+}(1, 0) - 1| < \epsilon \implies  1-\epsilon < F^{+}(1, 0) < 1+\epsilon\\
    & |F^{+}(1, 1)| < \epsilon. 
\end{align*}
We can pick any $\epsilon<0.5$ and have
\begin{equation}
\begin{cases}
    F^{+}(0, 0) < F^{+}(1, 0)\\
    F^{+}(1, 1) < F^{+}(1, 0)
\end{cases}   
\end{equation}
This is obviously contradictory to the fact that $F^{+}$ is is order-preserving monotone function, that is: 
\begin{equation}\label{eq:xor-F}
    F^{+}(0,0)\leqslant F^{+}(1,0) \leqslant F^{+}(1,1)
\end{equation}
The same logic applies to $\epsilon>0.5$, and for our proof we only need to find one such $\epsilon$. Therefore, $\exists\epsilon>0$, such that for all $F^+$ we cannot approximate XOR $f$ $\epsilon$-well.  
\end{proof}

Therefore, there does not exist an $F^{+}$ that can approximate XOR $f$ arbitrarily well. Thus $F^{+}$ is not a universal approximator.
\end{proof}

From Theorem~\ref{lma3}, we have that DNNs with non-negative weights are not universal approximators – they cannot even solve XOR. This is because they are order-preserving monotone functions. Can we overcome this limitation by flipping the sign of a single weight to make the network XOR-capable? The answer is \emph{yes}, and we illustrate this result in Figure~\ref{fig:0}, panel C: in a 3-unit single hidden layer network, flipping one output edge of one hidden unit negates the first quadrant, thus sculpting it out of the pink class region and joining it with the third quadrant. 

\subsection{Implication in neuroscience} 
How does our discussion on XOR pertain to the brain? Let's think about a simple discrimination task where the animal has to make a decision based on an input stimuli that has two features $x\in \{0,1\}^2$. For example, a deer has to tell toxic leaves apart from edible leaves. Due to seasonal changes, toxic leaves could be either green straight ($(0,0)$) or red curves ($(1,1)$); edible leaves could be either green curved ($(0,1)$) or red straight ($(1,0)$). To survive, the deer must be able to discriminate between toxic and edible leaves where the decision boundary follows exactly the XOR pattern. In fact, any discrimination task taking stimuli of the form $x\in \{0,1\}^2$ where the decision is dependent on both features is an XOR task \cite{Anderson2006}. Only after inhibitory connections evolved could organisms survive in more complex environments and perform more complex tasks. 

\section{Geometric intuitions of $\text{DNN}^{+}$ in $\mathbb{R}^n$}\label{sec:geometry}
So far we have proven feed-forward, fully connected neural networks without any negative weights are not universal approximators (Theorem~\ref{lma3}) due to their order-preserving nature (Theorem~\ref{lma2}). Next we will hone in on the concept of limited representation capacity by delineating three types of classification problems that $\text{DNN}^{+}$s fail to solve (Figure~\ref{fig:0} left column). These will also provide some geometric intuition on order-preserving functions. We summarize all theoretical results in $\mathbb{R}^2$ in Figure~\ref{fig:0} for intuitive comprehension. We prove all cases in finite dimensions $\mathbb{R}^n, n\in\mathbb{Z}^+$ in the three corollaries below. 


\begin{enumerate}
    \item Corollary~\ref{cor:1} notes that $\text{DNN}^{+}$s cannot solve a classification problem that requires the decision boundaries to have segments with a positive slope ($\mathbb{R}^2$), or a section of its decision hyperplane having a normal with all positive elements ($\mathbb{R}^n$). This is illustrated in Figure~\ref{fig:0} panel A). This follows from the fact that positive weights can only form decision boundaries that have negative slopes.
    \item Corollary~\ref{cor:2} notes that $\text{DNN}^{+}$s cannot solve classification problems where there exists a class whose decision boundary forms a closed shape ($\mathbb{R}^2$), or a closed region ($\mathbb{R}^n$). Order-preserving means that for all units in the non-negative DNNs, their activation gradients point towards the positive directions of all dimensions (top and/or right in $\mathbb{R}^2$). However, for a closed-shape decision boundary, it requires the gradient to point in opposite directions (both towards and away from the partition), which is not doable with non-negative DNNs.
    \item Corollary~\ref{cor:3} notes that $\text{DNN}^{+}$s cannot solve classification problems where the partition formed by the decision boundaries results in a disconnected set for one class (path-disconnected regions in the input feature space). This is a generalized topological explanation of why $\text{DNN}^{+}$s cannot solve XOR. 
\end{enumerate}

\ifthenelse{\equal{\papertype}{neurodata}}{
\begin{figure}[H]
\begin{center}
\includegraphics[width=\linewidth]{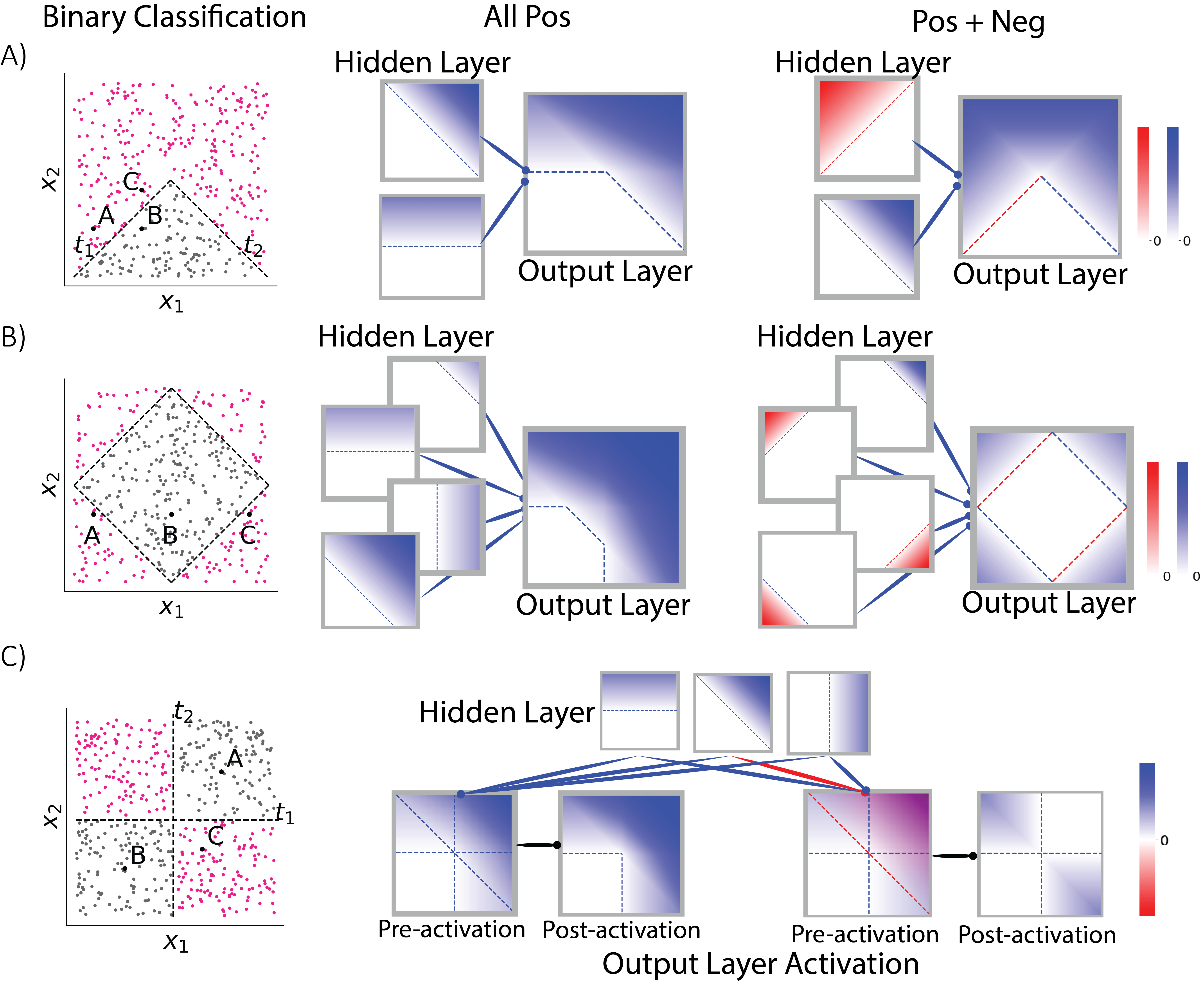}
\end{center}
}{
\begin{figure*}
\begin{center}
\includegraphics[width=.8\linewidth]{fig0}
\end{center}
}
\caption{\textbf{DNNs with only non-negative weights ($\text{DNN}^{+}$s) are NOT universal approximators.} Three problems not solvable by $\text{DNN}^{+}$ are presented here. All red colors in the plots indicate changes brought by flipping some weights' polarities to negative such that the binary classification problems on the left can be solved by the DNN. A) $\text{DNN}^{+}$s cannot solve problems where the decision boundaries contain any segment of positive slope (Corollary~\ref{cor:1}). Take the 2-hidden-unit network as an example: by flipping a single input weight to negative (notice red decision boundaries), the problem on the left becomes solvable.  B) $\text{DNN}^{+}$s cannot solve binary classification problems where there exists a decision boundary that forms a closed shape (Corollary~\ref{cor:2}). Take the 4-hidden-unit network as an example: flipping a single input weight for each of the two middle units allows decision boundaries of positive slopes to form, and further, flipping both input weights of the bottom unit allows the activation gradient to flow in the opposite direction (opposite to the order-preserving gradient, which is always to the top and/or right). These changes collectively make the closed-shape problem solvable. C) $\text{DNN}^{+}$s cannot solve binary classification problems where the partition formed by the decision boundaries results in a disconnected set for one class (Corollary~\ref{cor:3}), e.g., XOR (Theorem~\ref{lma3}). By flipping a single output weight of a single hidden unit, the top right quadrant can be sculpted out of the pink class, making the network XOR-solvable.}
\label{fig:0}
\ifthenelse{\equal{\papertype}{neurodata}}{
\end{figure}
}{
\end{figure*}
}

\subsection{A formal setup of decision boundary and partition of set}\label{sec:partition} 

In section~\ref{sec:geometry}, we focus on binary classification problems with $n$ input features. This means our target function is in the form $f:K\to \{0,1\}$, where $K$ is a compact domain and $K\subseteq\mathbb{R}^n$. The three corollaries can be easily extended into multi-class scenarios: whenever there exist two classes that satisfy the corollary statements, then the entire classification problem cannot be solved by $\text{DNN}^{+}$.

Instead of taking the angle of function approximation, in this section we set up classification problem $f:K\to \{0,1\}$ as a partition $P$ of compact domain $K$. 
\begin{definition}[Regions]
    $P$ is a partition of the compact domain $K$, where $P = \{R:R\subseteq K\subseteq\mathbb{R}^n\}$. We call the elements of $P$ regions. A single class in the classification problem can be the union of one or multiple regions. Further, for $P$ to be a partition of $K$, it satisfies these properties:
    \begin{enumerate}
        \item The family P does not contain the empty set, that is $\emptyset\notin P$;
        \item The union of the regions is $K$, that is $\cup\{R\}=K$;
        \item All regions are pair-wise disjoint, that is $\forall R_1, R_2\in P,\; R_1\neq R_2\implies R_1\cap R_2=\emptyset$;
        \item Within each region $f$ is constant, i.e., $\forall R\in P, \forall x\in R, f(x)=c$ where $c$ is a constant. For example, in our binary classification case, $c\in\{0,1\}$;
        \item Two adjacent regions belong to different classes. 
    \end{enumerate}
\end{definition}

\begin{definition}[Decision boundary]
    For each region $R$ in the partition $P=\{R:R\subseteq K\}$, the boundary of region $R$ is given by 
    \ifthenelse{\equal{\papertype}{neurodata}}{
    \begin{equation}
        \partial R = \{x\in R: \text{for every neighborhood}\; O \;\text{of}\; x,O\cap R\neq\emptyset\;\mathrm{and}\;O\cap(K\textbackslash R)\neq\emptyset\}.  
    \end{equation}    
    }{
    \begin{multline}
        \partial R = \{x\in R: \text{for every neighborhood}\; O \;\text{of}\; x,\\
        O\cap R\neq\emptyset\;\mathrm{and}\;O\cap(K\textbackslash R)\neq\emptyset\}.  
    \end{multline}    
    }
    The collection of the boundaries of all regions form the \textbf{decision boundaries} of the classification problem $f$, and is denoted by $\{\partial R\}$. 
\end{definition}

\begin{remark}[Neighborhood of the decision boundary]\label{rm:neighborhood}
Since two adjacent regions belong to different classes, then we must have points within the neighborhood of the decision boundary that belong to two different classes. 

Consider the $\epsilon$-neighborhood of a point $d$ on the decision boundary, i.e., for a point $d\in\partial R$, consider its $\epsilon$-neighborhood $O=\{x\in K:|x-d|<\epsilon, \epsilon>0\}$. Per the definition of boundary and the last property of region definition, we must have $\exists A,B\in O$, such that $f(A)\neq f(B)$. 
\end{remark}

\begin{remark}[Piecewise linear approximation of the decision boundary]\label{remark:epsilon}
Since we can approximate any function with piecewise linear functions up to some arbitrary accuracy \cite{Huang2020}, here we approximate the decision boundaries $\{\partial R\}$ with a family of linear functions $\mathbb{L}=\{L\}$, where $L=\{x\in K'\subseteq K:\mathbf{ax}+b=0, \mathbf{a}\in\mathbb{R}^n, b\in\mathbb{R}\}$. $\{L\}$ are connected line segments in $\mathbb{R}^2$ or connected hyperplane sections in higher dimensions $\mathbb{R}^n$. We always have $L\subseteq\mathbb{R}^{n-1}$, and $\mathbf{a}$ is the normal to the hyperplane $L$. Within the $\epsilon$ region of the boundary, based on the last property of the partition $\{R\}$, we have the following: 
\ifthenelse{\equal{\papertype}{neurodata}}{
\begin{equation}
    \begin{cases}
        f(x)=c_1, \forall x\in\{x\in K:0<\mathbf{ax}+b<\epsilon\} \\
        f(x)=c_2, \forall x\in\{x\in K:-\epsilon<\mathbf{ax}+b<0\} 
    \end{cases}, c_1\neq c_2, \epsilon>0 \text{ is small}\footnotemark
\end{equation}   
}{
\begin{multline}
    \begin{cases}
        f(x)=c_1, \forall x\in\{x\in K:0<\mathbf{ax}+b<\epsilon\} \\
        f(x)=c_2, \forall x\in\{x\in K:-\epsilon<\mathbf{ax}+b<0\} 
    \end{cases}, \\
    c_1\neq c_2, \epsilon>0 \text{ is small}
\end{multline}
}
\footnotetext{$\epsilon$ should be small such that $x$ is within the immediately adjacent classes $c_1\,\&\,c_2$; it should also be \emph{larger} than the linear approximation error region $|\mathcal{L}-\partial R|$. We assume in our paper the classification problems are well-behaved enough that they can be reasonably well approximated by piecewise linear functions and we can always find such $\epsilon$.}
\end{remark}

On a side note, decision boundaries are essentially the discontinuities of $f$ where the outputs of $f$ switch between $0$ and $1$. Although the original universal approximator theorem proofs require $f$ to be continuous, readers should not be too concerned on the continuity of $f$ for the following reason. We could easily relax the decision boundaries from line into a band where $f$ gradually switches between $0$ and $1$ in this band. Such relaxation does not alter any of our conclusions in the following 3 corollaries as we can always limit the bandwidth to be smaller than $\epsilon$ so that all of the above definitions still hold. 


\subsection{Geometric intuitions of $\textbf{DNN}^\mathbf{+}$}
\begin{corollaryE}[lma2][end]\label{cor:1}
\textup{(Boundary orientation, in $\mathbb{R}^n$).} 
DNNs with only non-negative weights cannot solve classification problems where the decision boundaries $\{L\}$ have any segment $L$ with a normal $\mathbf{a}=(a_1,\dots,a_n)$ where 
$
    \exists\;i\neq j\in [n], \; a_ia_j<0
$, 
i.e., positive slope in $\mathbb{R}^2$.
\end{corollaryE}
\begin{proofE}
Without loss of generality, we assume the feature index pair $i,j\in[n]$ satisfy $a_ia_j<0$. Let $\mathbf{d} = (d_1,\dots,d_i,\dots,d_j,\dots,d_n)$ be a point on the segment, i.e., $\mathbf{ad}+b=0$. Now we construct three points $A,B,C$ with  $\epsilon>0$.\footnote{For the choice of $\epsilon$ under non-linear decision boundary scenario, we assume here we can always find $\epsilon$ such that it is larger than the linear approximation error. For more details, please see Remark~\ref{remark:epsilon}} 
\begin{alignat}{3}
    A = &(d_1,\dots,\quad &&d_i-\epsilon,\dots,d_j,& &\dots,d_n)\\
    B = &(d_1,\dots,\quad &&d_i+\epsilon,\dots,d_j,& &\dots,d_n)\\
    C = &(d_1,\dots,\quad &&d_i+\epsilon,\dots,d_j-2\frac{a_i}{a_j}\epsilon,& &\dots,d_n)
\end{alignat}
First since $a_ia_j<0$ and $\epsilon>0$, we have $2\frac{a_i}{a_j}\epsilon<0$, thus 
\ifthenelse{\equal{\papertype}{neurodata}}{
\begin{equation}\label{eq:order-preserve}
    \begin{cases}
        x_i^A < x_i^B = x_i^C \\
        x_j^A = x_j^B < x_j^C
    \end{cases}\implies A\preceq B\preceq C \implies F^+(A)\leq F^+(B)\leq F^+(C)
\end{equation}
}{
\begin{multline}
    \label{eq:order-preserve}
    \begin{cases}
        x_i^A < x_i^B = x_i^C \\
        x_j^A = x_j^B < x_j^C
    \end{cases}\implies A\preceq B\preceq C \\
    \implies F^+(A)\leq F^+(B)\leq F^+(C)
\end{multline}
}

Simultaneously, from $\mathbf{ad}+b=0$, we also have
\begin{equation*}
    \begin{cases}
        \begin{alignedat}{3}
            &\mathbf{a}A+b=\mathbf{ad}+b -a_i\epsilon &&= -& &a_i\epsilon\\
            &\mathbf{a}B+b=\mathbf{ad}+b +a_i\epsilon &&= & &a_i\epsilon\\
            &\mathbf{a}C+b=\mathbf{ad}+b +a_i\epsilon-2\frac{a_i}{a_j}a_j\epsilon &&= -& &a_i\epsilon   
        \end{alignedat}
    \end{cases} 
\end{equation*}

Then $A,C$ must lie on the same side of $L$ but different than $B$, thus $f(A)=f(C)\neq f(B)$. By the same logic as in Theorem~\ref{lma3}, this contradicts with Eq~\ref{eq:order-preserve}; therefore, $\text{DNN}^{+}$ cannot solve classification problems where the decision boundaries $\{L\}$ have any segment $L$ with a normal $\mathbf{a}=(a_1,\dots,a_n)$ where $\exists\;i\neq j\in [n], a_ia_j<0$.
\end{proofE}

\begin{corollaryE}[lma2][end]\label{cor:2}
\textup{(Closed shape, in $\mathbb{R}^n$)} 
DNNs with only non-negative weights cannot solve binary classification problems where there exists a regions $R$ that is a closed set, i.e., the decision boundaries form a closed shape in $\mathbb{R}^2$. 
\end{corollaryE}
\begin{proofE}
Without loss of generality, let's assume region $R_0\in\{R\}$ is a closed set. We denote all points in $R_0$ as a general form $\mathbf{x}=(x_1,\dots,x_n)$. Consider any point $B=(x_1^B,\dots,x_n^B)\in R_0$, we can always find two points $A',C'\in\partial R$ that follow $A'\preceq B\preceq C'$ with the following construction method:
\begin{align*}
    A' = (min(x_1), x_2^B,\dots,x_n^B)\\
    C' = (max(x_1), x_2^B,\dots,x_n^B)
\end{align*}
As we move $\epsilon>0$ away from the boundary, we can further construct two points $A,C\notin R_0$ where
\begin{align*}
    A = (min(x_1)-\epsilon, x_2^B,\dots,x_n^B)\\
    C = (max(x_1)+\epsilon, x_2^B,\dots,x_n^B)
\end{align*}
Thus $A\preceq B\preceq C \implies F^+(A)\leqslant F^+(B) \leqslant F^+(C)$, yet we have $B\in R_0$ while $A,C\notin R_0 \implies f(A)=f(C)\neq f(B)$. By the same reasoning as in theorem~\ref{lma3}, we have a contradiction thus a $\text{DNN}^{+}$ cannot solve problems where the decision boundary forms a closed set.
\end{proofE}


\begin{definition}[path-disconnected point pair]
Suppose that $\mathcal X$ is a topological space. For a pair of points $x_1, x_2\in \mathcal X$, $x_1$ and $x_2$ are path-disconnected if there does \emph{not} exists a continuous function (path) $f:[0,1] \mapsto \mathcal X$ where $f(0) = x_1, f(1) = x_2$.
\end{definition}

\begin{definition}[disconnected space in $\mathbb{R}^2$]\label{def:disconnected-space}
A space $\mathcal X \subset \mathbb{R}^2$ is path-disconnected if there exists $x_1, x_2\in \mathcal X$ such that $x_1$ and $x_2$ are path-disconnected. \footnote{An example of disconnected space in $\mathbb{R}^2$ is the second and fourth quadrant that form class 0 in XOR.}
\end{definition}

\begin{corollaryE}[lma2][end]\label{cor:3}
\textup{(Disconnected space, in $\mathbb{R}^n$.)} 
DNNs with non-negative weights cannot solve a binary classification problem where there exists a class that is a disconnected space.
\end{corollaryE}
\begin{proofE}
Without loss of generality, we assume $R_0, R_1$ are disconnected and belong to the same class, i.e., $f(x)=c_1, \forall x\in R_0\cup R_1$, $c_1$ is a constant. 
Now consider any pair of points $(A,B), A\in R_0, B\in R_1$, the straight line segment $AB$ that connects $A$ and $B$ must pass through another class by Definition~\ref{def:disconnected-space}. This means we must have point $C$ on line $AB$, but $f(c)=c_2\neq c_1$ where $c_2$ is a constant. 
Next, we discuss the order relationship between $A$ and $B$:
\paragraph{case 1} \textbf{Exists such a pair} $\mathbf{A\preceq B}$. Then since $A,C,B$ are colinear and C is in between $A$ and $B$, we have $A\preceq C\preceq B\implies F^+(A)\leqslant F^+(C) \leqslant F^+(B)$. Yet by construction, we also have $f(A)=f(B)\neq f(C)$. By the same reasoning as in Theorem~\ref{lma3}, we have a contradiction. Thus, a $\text{DNN}^{+}$ cannot solve classification problems that fall into this case. 
\paragraph{case 2} \textbf{Does NOT exist such a pair} $\mathbf{A\preceq B}$. This means for all pairs of points in the two disconneted regions, they don't follow the ordering defined in Definition~\ref{def:order}. Thus, there exists two input dimensions, $i,j\in [n]$, such that for all $\mathbf{x^0}=(\dots,x_i^0,\dots,x_j^0,\dots)\in R_0$, and for all $\mathbf{x^1}=(\dots,x_i^1,\dots,x_j^1,\dots)\in R_1$, they follow
\begin{equation}
    \begin{cases}
        x_i^0 < t_i < x_i^1\\
        x_j^0 > t_j > x_j^1
    \end{cases}, t_1,t_2\in\mathbb{R}
\end{equation}
Now, for a point $G=(\dots,x_i^F,\dots,x_j^F,\dots)\in R_0$, we always have $x_i^G<t_i\; \mathrm{and}\; x_j^G>t_j$. Further, we can always construct two more points $D,E\in (K-R_0-R-1)$ by
\begin{align*}
    &D = (\dots,x_i^G,\dots,t_j,\dots)\\
    &E = (\dots,t_i,\dots,x_j^G,\dots)
\end{align*}
Thus we have $D\preceq G\preceq E\implies F^+(D)\leqslant F^+(G) \leqslant F^+(E)$. However, since we have $F\in R_0$ and yet $D,E\in (K-R_0-R-1)$, we therefore have $f(D)=f(E)\neq f(G)$. By the same reasoning as in theorem~\ref{lma3}, we have a contradiction. Thus, a $\text{DNN}^{+}$ cannot solve classification problems that fall into this case. 

Collectively, we proved that a $\text{DNN}^{+}$ cannot solve a classification problem where there exists a class that is a disconnected space.
\end{proofE}

These three corollaries point to the fundamental flaw of $\text{DNN}^{+}$s: however many layers they have, each penultimate layer unit can only form a single continuous decision boundary that is composed of segments having negative slopes (or composed of hyperplanes having all-positive normal). Adding more layers or adding more nodes to each layer of such a network can produce more complex-shaped decision boundaries (Figure~\ref{fig:0}, panel B, middle column), but cannot form boundaries of more orientations, or form a closed region, or sculpt the input space such that disconnected regions can be joined. Therefore, taking negative weights away from DNNs drastically shrinks their repertoire of representable functions. For real-world problem solving, it is crucial to have both positive and negative weights in the network.

\section{Extension to convolutional neural networks}\label{sec:extension}
In this section, we will discuss and prove that our theoretical results in section~\ref{sec:app} and~\ref{sec:geometry} are generalizable to many forms of DNN that are variants of the MLP definition (Def~\ref{def:DNN}). We are particularly interested in the family of convolutional neural networks (CNNs) since their activity space closely resembles those observed in the visual cortex and auditory cortex in the brain, as shown by various correlation studies~\cite{Bashivan2019, Alexander2018}. 

Essentially, substituting matrix multiplication with \textbf{convolution} (Definition~\ref{def:conv}) and adding additional \textbf{max-pooling} layers (Definition~\ref{def:pool}) makes a CNN \cite{Lecun1989, Lecun1998, Alex2012}. The universality of CNNs has been explicitly proven~\cite{ZHOU2020, Cohen2016, Poggio2017, Heinecke2020}. Additionally, one can further add \textbf{skip connections} (Definition~\ref{def:skip}) to make a residual network~\cite{He2016}. 

We will discuss how having convolution layers and adding skip connections make DNNs equivalent to our original MLP definition and thus subject to the same limited representation power when all weights are non-negative; we will also show our results are generalizable to DNNs with additional max-pooling layers since the order-preserving property of $\text{DNN}^{+}$ holds in general. 

\subsection{Convolution layer} 
The convolution operation between an input (normally an image) and a single filter is defined below. For the convolutional layers in CNN, instead of matrix multiplication, multiple filters convolve the input, and the values in the filters are the weights in CNN. It can be shown that all convolution operations can be converted into matrix multiplications~\cite{chellapilla2006}, where all the filters are converted into the form of Toeplitz type matrices and all images are vectorized. For more details, we refer readers to the universality proofs of CNN~\cite{ZHOU2020}. Therfore, a convolutional neural network can be converted into a MLP (Definition~\ref{def:DNN}). Therefore, for a convolutional neural network with all non-negative weights, all of our results in this paper hold. 
\begin{definition}[Convolution 2D]\label{def:conv}
    The convolution $O:\mathbb{R}^{M\times N}\times\mathbb{R}^{m\times n} \to \mathbb{R}^{(M-m+1)\times (N-n+1)}$ between an input (image) $I\in\mathbb{R}^{M\times N}$ and a kernel (feature map) $K\in\mathbb{R}^{m\times n}$ is given by
    \ifthenelse{\equal{\papertype}{neurodata}}{
    \begin{equation}
        O(i,j) = \sum^{m}_{k=1}\sum^{n}_{l=1}I(i+k-1, j+l-1)K(k,l), i\in [M-m+1],\; j\in [N-n+1]    
    \end{equation}
    }{
    \begin{multline}
        O(i,j) = \sum^{m}_{k=1}\sum^{n}_{l=1}I(i+k-1, j+l-1)K(k,l), \\
        i\in [M-m+1],\; j\in [N-n+1]      
    \end{multline}   
    }
\end{definition}

\subsection{Max pooling layer} 
Since max pooling functions are order-preserving monotone functions as well, by the closure of monotone functions under compositionality, we further have DNNs with max pooling layers follow all the results in the previous sections. 
\begin{definition}[Max-pooling]\label{def:pool}
    Max-pooling functions are $pool:\mathbb{R}^n\to\mathbb{R}, \;pool(\mathbf{x})=max(x_i)$, where $\mathbf{x}=(x_1,\dots,x_i,\dots,x_n)$.
\end{definition}
\begin{lemma}[Max pooling are order preserving]\label{lma:pooling}
    Max-pooling functions are order-preserving monotonone functions. Equivalently,
    \begin{equation}
        \mathbf{x}\preceq \mathbf{y} \implies pool(\mathbf{x})\leqslant pool(\mathbf{y})
    \end{equation}
\end{lemma}
\begin{proof}
Without loss of generality, for any pair of points $\mathbf{x}=(x_1,\dots,x_i,\dots,x_n), \mathbf{y}=(y_1,\dots,y_i,\dots,y_n) \in \mathbb{R}^n$ that follow the order $\mathbf{x}\preceq \mathbf{y}$, assume their max pooling output is given by the p\textsuperscript{th} and q\textsuperscript{th} dimensions, respectively.  Then we have
\begin{align*}
    &\forall i\in [n], x_i\leqslant x_p\\
    &\forall j\in [n], y_j\leqslant y_q.
\end{align*}
Also by the ordering $\mathbf{x}\preceq \mathbf{y}$, we have $y_p\geqslant x_p$, which means 
\begin{equation}
    pool(\mathbf{x})=x_p\leqslant y_p\leqslant y_q=pool(\mathbf{y}).
\end{equation}
Thus max pooling functions are order-preserving monotone functions. 
\end{proof}

\subsection{Skip connections} 
A skip connection connects two non-adjacent layers, e.g., layers $(l-2)$ and $l$, through identity mapping (Definition~\ref{def:skip}). 
\begin{definition}[Skip connections]\label{def:skip}
    Each layer with an incoming skip connection is given by $f^{(l)}: \mathbb{R}^{n_{(l-1)}}\to\mathbb{R}^{n_{l}}$: 
    \begin{equation*}
        f^{(l)}(\mathbf{x})=\mathbf{\Phi}(\mathbf{W}^{(l)}\mathbf{x}+\mathbf{b}^{(l)})\underline{+f^{(l')}(\mathbf{x})}, \quad l'\in\{0,\dots,l-2\}
    \end{equation*}
\end{definition}
We can always convert a DNN with skip connections into a classical MLP (Definition~\ref{def:DNN}) by adding dummy units in the intermediate layers (in $(l-1)$ and $l$) and setting their incoming and outgoing weights to match the identity mapping of the skip connection, i.e., weights of each unit only having a single $1$ entry where the skip connection is and $0$ elsewhere. Therefore, for a DNN with skip connections, all of our results in this paper hold. 

\section{Conclusion and Discussion}\label{sec:discussion}
We proved that DNNs with all non-negative weights (i.e., without inhibitory connections) are not universal approximators (Theorem~\ref{lma2}). This is because non-negative DNNs are exclusively order-preserving monotone functions (Theorem~\ref{lma1}). Some geometric implications of this property in the finite euclidean space $\mathbb{R}^n$ are proved in Corollaries~\ref{cor:1}-\ref{cor:3}. Specifically, each output unit in a network without inhibitory connections can only form a single continuous decision boundary that is composed of hyperplane sections having a very particular orientation (hyperplane with all-positive normal). Intuitively in $\mathbb{R}^2$ input space, this means only forming a continuous line composed of segments having negative slopes. The addition of inhibitory connections to the networks allows more complex boundaries to form (e.g., boundaries of positive orientations (Corollary~\ref{cor:1}) and of closed shapes (Corollary~\ref{cor:2})); the addition of inhibition also allows for sculpting/folding of the representation space (Corollary~\ref{cor:3}). Together, these results prove that both DNNs and brains, which can be abstracted as networks with non-decreasing monotone activation functions,  need inhibitory connections to learn more functions. 

How translatable are our theoretical results on DNNs to brains? Under the assumption that the activation functions  of all neurons are non-decreasing, our theoretical results directly shed light onto the long-standing question of why brains have inhibitory connections. We recognize there are types of questions that cannot be answered by DNN theory. For example, the learning process of DNNs differs from biological systems \cite{Lillicrap2016},  DNNs do not follow Dale's principle \cite{Cornford2020}, units in DNNs do not spike, etc. We emphasize our proof does 
\emph{not} rely on any assumption that involves the above discrepancies between DNNs and brains; instead, our proof only relies on the non-decreasing activation function assumption. The simplicity of our assumption is the sole reason behind why our results in general hold for many forms of DNN that have been experimentally shown to resemble the brain~\ref{sec:extension}. Our answer from a representation-capacity point of view supplements the dynamic system story of E/I balance and provides new perspectives to these long-standing neuroscience questions.

Our current work is just a first step in understanding the representation space of networks from the lens of connection polarity. What we have proven is a first-order property of the network concerning the \emph{existence} of negative connections. The next step is to look at the second-order property that concerns the \emph{configuration} of connection polarities. For this second-order property, a strict non-negativity constraint is no longer imposed on the network weights; instead, we constrain the connections to follow a specific polarity configuration, or in the neuroscience language - a circuit rule. We will illustrate the feasibility of this idea and its significance with an example: in the cortex, excitatory and inhibitory neurons connect / synapse in very different manners, and they follow a very stereotypical pattern: \textit{local excitation, broad inhibition}. Such a configuration principle has been suggested to be the underlying mechanism of surround inhibition ~\cite{Angelucci2017, Ozeki2009}, an important computation process that allows for context-dependent activation and redundancy reduction. Based on our Corollaries~\ref{cor:1}-\ref{cor:3}, it is very likely that in a local subspace, the largely excitatory sub-network is order-preserving monotone and only forms continuous decision boundaries of negative slopes; on a global level, such communities of subspaces are connected through inhibitory connections and collectively form complex decision boundaries. It is an exciting future direction to connect these geometric properties with the relatively better-understood functional significance of neural circuits. Similar ideas can be explored in the canonical circuits in decision making \cite{Mysore2020} (e.g., winner take all). How circuit rules translate into the geometric constraints on the representation space has been largely unexplored. It was a route more regularly taken in the earliest era of DNNs (perceptrons, specifically) ~\cite{Minsky1969} and became increasingly more challenging as the size of the network grew; we look forward to building on top of the theoretical work presented in this paper and bridging the gap between network topology (E-I configuration) and the associated representation space. 

\medskip
\section*{Acknowledgements}
This work is supported by grants awarded to J.T.V:  \href{https://www.nsf.gov/awardsearch/showAward?AWD_ID=1942963&HistoricalAwards=false}{NSF CAREER Award (Grant No. 1942963)},  \href{https://www.nsf.gov/awardsearch/showAward?AWD_ID=2014862}{NSF NeuroNex Award (Grant No. 2014862)}, \href{https://www.nsf.gov/awardsearch/showAward?AWD_ID=2020312&HistoricalAwards=false}{NSF AI Institute Planning Award (Grant No. 2020312)}, and \href{https://www.nsf.gov/awardsearch/showAward?AWD_ID=2031985&HistoricalAwards=false}{NSF Collaborative Research: THEORI Networks (Grant No. 2031985)}. 
Q.W. was also partially supported by \href{https://reporter.nih.gov/search/uO0lXp3vF0iu1clRXTJOQg/project-details/10203331}{Johns Hopkins neuroscience T32 training grant (Grant no. 091018)}. 
The authors would like to thank the \href{https://neurodata.io/about/team/}{NeuroData} lab members for helpful feedback.

\ifthenelse{\equal{\papertype}{neurodata}}{
\bibliographystyle{unsrt}
\bibliography{ref}
}{
{\small
\bibliographystyle{ieee_fullname}
\bibliography{ref}

\begin{thebibliography}{10}

\bibitem{Lim2013}
Sukbin Lim and Mark~S. Goldman.
\newblock {Balanced cortical microcircuitry for maintaining information in
  working memory}.
\newblock {\em Nature Neuroscience}, 16(9):1306--1314, 2013.

\bibitem{Lam2022}
Norman~H. Lam, Thiago Borduqui, Jaime Hallak, Antonio Roque, Alan Anticevic,
  John~H. Krystal, Xiao~Jing Wang, and John~D. Murray.
\newblock Effects of altered excitation-inhibition balance on decision making
  in a cortical circuit model.
\newblock {\em Journal of Neuroscience}, 42, 2022.

\bibitem{Vogels2011}
T.~P. Vogels, H.~Sprekeler, F.~Zenke, C.~Clopath, and W.~Gerstner.
\newblock {Inhibitory plasticity balances excitation and inhibition in sensory
  pathways and memory networks}.
\newblock {\em Science}, 334(6062):1569--1573, dec 2011.

\bibitem{Sukenik2021}
Nirit Sukenik, Oleg Vinogradov, Eyal Weinreb, Menahem Segal, Anna Levina, and
  Elisha Moses.
\newblock Neuronal circuits overcome imbalance in excitation and inhibition by
  adjusting connection numbers.
\newblock {\em Proceedings of the National Academy of Sciences of the United
  States of America}, 118, 2021.

\bibitem{Cohen2002}
Ivan Cohen, Vincent Navarro, St{\'{e}}phane Clemenceau, Michel Baulac, and
  Richard Miles.
\newblock {On the origin of interictal activity in human temporal lobe epilepsy
  in vitro}.
\newblock {\em Science}, 298(5597):1418--1421, nov 2002.

\bibitem{Huberfeld2011}
Gilles Huberfeld, Liset {Menendez De La Prida}, Johan Pallud, Ivan Cohen,
  Michel {Le Van Quyen}, Claude Adam, St{\'{e}}phane Clemenceau, Michel Baulac,
  and Richard Miles.
\newblock {Glutamatergic pre-ictal discharges emerge at the transition to
  seizure in human epilepsy}.
\newblock {\em Nature Neuroscience}, 14(5):627--635, 2011.

\bibitem{Truccolo2011}
Wilson Truccolo, Jacob~A. Donoghue, Leigh~R. Hochberg, Emad~N. Eskandar,
  Joseph~R. Madsen, William~S. Anderson, Emery~N. Brown, Eric Halgren, and
  Sydney~S. Cash.
\newblock {Single-neuron dynamics in human focal epilepsy}.
\newblock {\em Nature Neuroscience}, 14(5):635--643, 2011.

\bibitem{Kristan2016}
William~B. Kristan.
\newblock Early evolution of neurons.
\newblock {\em Current Biology}, 26, 2016.

\bibitem{Hornik1989}
Kurt Hornik, Maxwell Stinchcombe, and Halbert White.
\newblock {Multilayer feedforward networks are universal approximators}.
\newblock {\em Neural Networks}, 2(5):359--366, jan 1989.

\bibitem{Cybenko1989}
G.~Cybenko.
\newblock {Approximation by superpositions of a sigmoidal function}.
\newblock {\em Mathematics of Control, Signals and Systems 1989 2:4},
  2(4):303--314, dec 1989.

\bibitem{Hornik1991}
Kurt Hornik.
\newblock {Approximation capabilities of multilayer feedforward networks}.
\newblock {\em Neural Networks}, 4(2):251--257, jan 1991.

\bibitem{Dickey1993}
F.~M. Dickey and J.~M. DeLaurentis.
\newblock Optical neural networks with unipolar weights.
\newblock {\em Optics Communications}, 101, 1993.

\bibitem{DeLaurentis1994}
J.~M. DeLaurentis and F.~M. Dickey.
\newblock A convexity-based analysis of neural networks.
\newblock {\em Neural Networks}, 7, 1994.

\bibitem{Hennie2010}
Hennie Daniels and Marina Velikova.
\newblock Monotone and partially monotone neural networks.
\newblock {\em IEEE Transactions on Neural Networks}, 21(6):906--917, 2010.

\bibitem{Amos2017}
Brandon Amos, Lei Xu, and J.~Zico Kolter.
\newblock Input convex neural networks.
\newblock In Doina Precup and Yee~Whye Teh, editors, {\em Proceedings of the
  34th International Conference on Machine Learning}, volume~70 of {\em
  Proceedings of Machine Learning Research}, pages 146--155. PMLR, 06--11 Aug
  2017.

\bibitem{Anderson2006}
Britt Anderson, Jessie~J. Peissig, Jedediah Singer, and David~L. Sheinberg.
\newblock Xor style tasks for testing visual object processing in monkeys.
\newblock {\em Vision Research}, 46, 2006.

\bibitem{Huang2020}
Changcun Huang.
\newblock {ReLU Networks Are Universal Approximators via Piecewise Linear or
  Constant Functions}.
\newblock {\em Neural Computation}, 32(11):2249--2278, 11 2020.

\bibitem{Bashivan2019}
Pouya Bashivan, Kohitij Kar, and James~J. DiCarlo.
\newblock Neural population control via deep image synthesis.
\newblock {\em Science}, 364, 2019.

\bibitem{Alexander2018}
Alexander~J.E. Kell, Daniel~L.K. Yamins, Erica~N. Shook, Sam~V.
  Norman-Haignere, and Josh~H. McDermott.
\newblock A task-optimized neural network replicates human auditory behavior,
  predicts brain responses, and reveals a cortical processing hierarchy.
\newblock {\em Neuron}, 98:630--644.e16, 2018.

\bibitem{Lecun1989}
Y.~LeCun, B.~Boser, J.~S. Denker, D.~Henderson, R.~E. Howard, W.~Hubbard, and
  L.~D. Jackel.
\newblock {Backpropagation Applied to Handwritten Zip Code Recognition}.
\newblock {\em Neural Computation}, 1(4):541--551, 12 1989.

\bibitem{Lecun1998}
Y.~Lecun, L.~Bottou, Y.~Bengio, and P.~Haffner.
\newblock Gradient-based learning applied to document recognition.
\newblock {\em Proceedings of the IEEE}, 86(11):2278--2324, 1998.

\bibitem{Alex2012}
Alex Krizhevsky, Ilya Sutskever, and Geoffrey~E Hinton.
\newblock Imagenet classification with deep convolutional neural networks.
\newblock In F.~Pereira, C.J. Burges, L.~Bottou, and K.Q. Weinberger, editors,
  {\em Advances in Neural Information Processing Systems}, volume~25. Curran
  Associates, Inc., 2012.

\bibitem{ZHOU2020}
Ding-Xuan Zhou.
\newblock Universality of deep convolutional neural networks.
\newblock {\em Applied and Computational Harmonic Analysis}, 48(2):787--794,
  2020.

\bibitem{Cohen2016}
Nadav Cohen and Amnon Shashua.
\newblock Convolutional rectifier networks as generalized tensor
  decompositions.
\newblock In Maria~Florina Balcan and Kilian~Q. Weinberger, editors, {\em
  Proceedings of The 33rd International Conference on Machine Learning},
  volume~48 of {\em Proceedings of Machine Learning Research}, pages 955--963,
  New York, New York, USA, 20--22 Jun 2016. PMLR.

\bibitem{Poggio2017}
Tomaso Poggio, Hrushikesh Mhaskar, Lorenzo Rosasco, Brando Miranda, and Qianli
  Liao.
\newblock Why and when can deep-but not shallow-networks avoid the curse of
  dimensionality: A review.
\newblock {\em International Journal of Automation and Computing}, 14, 2017.

\bibitem{Heinecke2020}
Andreas Heinecke, Jinn Ho, and Wen-Liang Hwang.
\newblock Refinement and universal approximation via sparsely connected relu
  convolution nets.
\newblock {\em IEEE Signal Processing Letters}, 27:1175--1179, 2020.

\bibitem{He2016}
Kaiming He, Xiangyu Zhang, Shaoqing Ren, and Jian Sun.
\newblock Deep residual learning for image recognition.
\newblock In {\em 2016 IEEE Conference on Computer Vision and Pattern
  Recognition (CVPR)}, pages 770--778, 2016.

\bibitem{chellapilla2006}
Kumar Chellapilla, Sidd Puri, and Patrice Simard.
\newblock {High Performance Convolutional Neural Networks for Document
  Processing}.
\newblock In Guy Lorette, editor, {\em {Tenth International Workshop on
  Frontiers in Handwriting Recognition}}, La Baule (France), October 2006.
  {Universit{\'e} de Rennes 1}, {Suvisoft}.
\newblock http://www.suvisoft.com.

\bibitem{Lillicrap2016}
Timothy~P. Lillicrap, Daniel Cownden, Douglas~B. Tweed, and Colin~J. Akerman.
\newblock Random synaptic feedback weights support error backpropagation for
  deep learning.
\newblock {\em Nature Communications}, 7:1--10, 2016.

\bibitem{Cornford2020}
Jonathan Cornford, Damjan Kalajdzievski, Marco Leite, Amélie Lamarquette,
  Dimitri~M Kullmann, Blake Richards, and Mila~/ Mcgill.
\newblock Learning to live with dale's principle: Anns with separate excitatory
  and inhibitory units.
\newblock {\em bioRxiv}, 2020.

\bibitem{Angelucci2017}
Alessandra Angelucci, Maryam Bijanzadeh, Lauri Nurminen, Frederick Federer, Sam
  Merlin, and Paul~C. Bressloff.
\newblock Circuits and mechanisms for surround modulation in visual cortex.
\newblock {\em Annual Review of Neuroscience}, 40:425--451, 2017.

\bibitem{Ozeki2009}
Hirofumi Ozeki, Ian~M. Finn, Evan~S. Schaffer, Kenneth~D. Miller, and David
  Ferster.
\newblock Inhibitory stabilization of the cortical network underlies visual
  surround suppression.
\newblock {\em Neuron}, 62:578--592, 2009.

\bibitem{Mysore2020}
Shreesh~P. Mysore and Ninad~B. Kothari.
\newblock Mechanisms of competitive selection: A canonical neural circuit
  framework.
\newblock {\em eLife}, 9, 2020.

\bibitem{Minsky1969}
Marvin Minsky and Seymour Papert.
\newblock {\em Perceptrons.}
\newblock M.I.T. Press, 1969.

\end{thebibliography}
}
}

\appendix
\newpage
\section{Appendix - Proofs}\label{sec:proof}
\printProofs
\end{document}